\documentclass[conference]{IEEEtran}

\usepackage{xspace,exscale,relsize}
\usepackage{fancybox,shadow}
\usepackage{graphicx}
\usepackage{color}
\usepackage{bm}
\usepackage{amsthm,amsfonts,amssymb,amsmath}
\usepackage{ifpdf}
\usepackage{subfigure}
\usepackage{prettyref,enumerate}
\usepackage{algorithm,algorithmic}
\usepackage{multirow}

\theoremstyle{definition}
 
\newtheorem{theorem}{Theorem}[section]

\newtheorem{definition}[theorem]{Definition}
\newtheorem{lemma}[theorem]{Lemma}

\theoremstyle{remark}

\newtheorem{remark}{Remark}[section]

\renewcommand{\leq}{\leqslant} 
\renewcommand{\geq}{\geqslant}

\newcommand{\cM}{\mathcal{M}}

\newcommand{\bE}{\mathbb{E}}

\newcommand{\bP}{\mathbb{P}}\newcommand{\bR}{\mathbb{R}}
\newcommand{\bS}{\mathbb{S}}

\DeclareMathOperator{\argmin}{argmin}

\begin{document}
%
\title{Multiclass MinMax Rank Aggregation}

\author{
  \IEEEauthorblockN{
    Pan~Li~
    and
    Olgica~Milenkovic \\
    }
  {\normalsize
    \begin{tabular}{ccc}
      ECE Department, University of Illinois at Urbana-Champaign \\
      Email: panli2@illinois.edu, milenkov@illinois.edu
    \end{tabular}}\vspace{-3ex}
    }
\maketitle


%

\begin{abstract}
We introduce a new family of minmax rank aggregation problems under two distance measures, the Kendall $\tau$
and the Spearman footrule. As the problems are NP-hard, we proceed to describe a number
of constant-approximation algorithms for solving them. We conclude with illustrative applications of the aggregation 
methods on the Mallows model and genomic data.
\end{abstract}

%
\IEEEpeerreviewmaketitle
\vspace{-0.05in}
\section{Introduction}
\vspace{-0.05in}
Rankings, a special form of ordinal data, have received significant attention in the machine learning community as they arise in a number of important application domains, such as recommender systems, social voting and product placement platforms. Of particular importance are rankings of the form of linear orders (permutations) and partial rankings (weak orders), which are frequently obtained through conversion from ratings. One of the main processing tasks for rankings is \emph{rank aggregation}, which often involves evaluating the median of a set of permutations or partial rankings under a suitably chosen distance function~\cite{ailon2008aggregating,bartholdi1989voting,diaconis1977spearman,dwork2001rankw,kemeny1959mathematics,kenyon2007rank,van2007deterministic}. The median rank aggregation problem under the Kendall $\tau$ distance was introduced by Kemeny~\cite{kemeny1959mathematics}, and was proved to be NP-hard by Bartholdi et al.~\cite{bartholdi1989voting}. A number of approximation algorithms for the problem have been described in~\cite{ailon2008aggregating}, mostly pertaining to permutations; a corresponding PTAS (polynomial time approximation scheme) was proposed in~\cite{kenyon2007rank}. In the context of partial ranking aggregation, known solutions include the results of~\cite{ailon2010aggregation,fagin2004comparing}. Median aggregation under other distance functions has received less attention, one notable exception being the Spearman rank aggregation problem~\cite{diaconis1977spearman}, which is known to provide a constant approximation for Kendall $\tau$ aggregation using a polynomial time algorithm based on weighted bipartite matching~\cite{dwork2001rankw}. 

We propose to investigate a broad new family of rank aggregation problems in which the median is replaced by a minmax type of function and where the rankings are grouped in classes. More precisely, assume that there are $C \geq 1$ different classes of rankings and let $\Sigma^k=\{\sigma_1^{k},\sigma_2^k,...,\sigma_{m_k}^k\}$ be the set of $m_k=|\Sigma^k|$ rankings belonging to the class labeled by $k \in [C]$. Our minmax rank aggregation problem may be succinctly described as follows:  Output a ranking $\pi$ that agrees in the minmax sense with the rankings belonging to the different classes. Rigorously, we seek to solve the following optimization problem:
\begin{align*}
\textbf{MinMax:} \quad \min_{\pi}\max_{k} \lambda_k \, d(\pi, \Sigma^k),
\end{align*}
where $\lambda_k>0$ represent the costs of violating the agreement with rankings in class $k$. In the above formulation, $d(\pi, \Sigma)$ stands for a distance between a ranking or partial ranking $\pi$ and a \emph{set of rankings} $\Sigma^k$, and it may be chosen to be of the form of a median distance (which equals the total sum of distances between $\pi$ and the elements of $\Sigma^k$) or a minimum distance (which equals the smallest distance between $\pi$ and an element in $\Sigma^k$). The above described \textbf{MinMax} problem is motivated by a number of applications in which classes of rankings arise due to different ranking criteria or properties of the ranking entities (social platforms) or due to prior knowledge of different similarity degrees in groups of rankings (genome evolution). The minmax criteria is typically used when trying to ensure that the aggregate violates each vote (class of votes) to roughly the same extent.

We start our analysis with the \textbf{MinMax} problem with $C=1$ and under the median and minimum distance, and then proceed to study the problem for the case of arbitrary values of $C$ and $m_k$, $k=1,\ldots,C$.
For both the case of the Kendall $\tau$ as well as the Spearman footrule in the median and minimum distance setting, the $\textbf{MinMax}$ problems may be shown to be NP-hard by using the corresponding results of~\cite{bachmaier2015hardness}. In particular, the work in~\cite{bachmaier2015hardness} outlines a general framework for proving NP-hardness results for the median, single class min-max-aggregation problem under different ranking distances. Nevertheless, only a handful of approximation algorithms were proposed even for this basic min-max-aggregation form: To the best of our knowledge, the only provable algorithm for the single class \textbf{MinMax} under the minimum distance measure was provided in~\cite{bachmaier2015hardness}. The algorithm takes the form of the well studied "pick-a-permutation'' method, and tends to perform poorly in practice. 

The main results of our work include families of constant approximation algorithm for the new, general family of multiclass \textbf{MinMax} problems, both under the median and minimum class distance, evaluated using the Kendall $\tau$ and Spearman footrule. Furthermore, we illustrate the use of the new aggregation paradigm on the problem of finding an ancestral genome arrangement for mitochondrial DNA under the tandem duplication model for genomes~\cite{chaudhuri2006tandem}.
\section{Mathematical Preliminaries} \label{sec:prelim}
Let $S$ denote a set of $n$ elements, which without loss of generality we set to $[n]\equiv\{{1,2,\ldots,n\}}$. 
A ranking is an ordering of a subset of elements $Q$ of $[n]$ according to a 
predefined rule. When $Q=[n]$, the resulting order is referred to as a permutation. When the rankings include ties, they are referred to as partial rankings~\cite{fagin2004comparing}.

More precisely, a permutation is a bijection $\sigma \, : \, [n] \rightarrow [n]$, and the set of permutations over $[n]$ forms the symmetric group of order $n!$, denoted by $\bS_n$. 
For any $\sigma \in \bS_n$ and $x \in [n]$, $\sigma(x)$ denotes the rank (position) of the element $x$ in $\sigma$. We say that $x$ is ranked higher than $y$ (ranked lower than $y$) iff $\sigma(x)<\sigma(y)$ ($\sigma(x)>\sigma(y)$). The inverse of a permutation $\sigma$ is denoted by $\sigma^{-1}: [n]\rightarrow [n]$. Clearly, $\sigma^{-1}(t)$ represents the element ranked at position $t$ in $\sigma$. Similarly, partial rankings~\cite{fagin2004comparing} represent a mapping over $[n]$ in which there may exist two elements $x \neq y$ such that $\sigma(x)=\sigma(y)$. It is common to use $\sigma(x)$ to denote the position of the element $x$ in the partial ranking $\sigma$, and to define it as
\begin{align} \label{eq:position}
\sigma(x)&\triangleq |\{y\in [n]: y\;\text{is ranked higher than}\;x\}| \notag \\
&+\frac{1}{2}(|\{y\in [n]: y\;\text{is tied with}\;x\}|+1).
\end{align}

A number of distance functions between rankings were proposed in the literature~\cite{diaconis1977spearman,fagin2004comparing,stanley2011enumerative}. One distance function counts the number of \emph{adjacent} transpositions needed to convert a permutation into another. Adjacent transpositions generate $\mathbb{S}_n$, i.e., any permutation $\pi\in\mathbb{S}_n$ can be converted into another permutation $\sigma\in\mathbb{S}_n$ through a sequence of adjacent transpositions~\cite{stanley2011enumerative}. The smallest number of adjacent transpositions needed to convert a permutation $\pi$ into another permutation $\sigma$ is termed the Kendall
$\tau$ distance, denoted by $d_{\tau}(\pi,\sigma)$. The Kendall $\tau$ distance between two permutations $\pi$ and $\sigma$ over $[n]$ also equals the number of pairwise inversions of elements of the two permutations:
\begin{align}\label{fullmetric}
d_{\tau}(\sigma,\pi)= |\{(x,y):  \pi(x) > \pi(y),\sigma(x) < \sigma(y)\}|.
\end{align}
Another \emph{positional} distance measure is the Spearman footrule, 
$$d_S(\sigma,\pi) = \sum_{x\in [n]} |\sigma(x) -\pi(x)|.$$ 
It can be shown that $d_{\tau}(\pi,\sigma) \leq d_S(\pi,\sigma) \leq 2 d_{\tau}(\pi,\sigma)$~\cite{diaconis1977spearman}.

One may similarly define a generalization of the Kendall $\tau$ distance for partial rankings $\pi$ and $\sigma$ over the set $[n]$.
This distance is known as the Kemeny distance, and equals 
\begin{align}\label{partialmetric}
d_{K}(\pi,\sigma)=&|\{(x,y): \pi(x)>\sigma(y),\pi(x)<\sigma(y)\}|\nonumber\\
+&\frac{1}{2}|\{(x,y): \pi(x)=\pi(y),\sigma(x)>\sigma(y) \nonumber \\
\text{\, or \,}&\;\pi(x)>\pi(y),\sigma(x)=\sigma(y)\}|.
\end{align}
The Spearman footrule analogue for partial rankings~\cite{fagin2004comparing} equals the sum of the absolute differences between ``positions'' of elements in the partial rankings,
\begin{align*}
d_{prS}(\sigma,\pi) = \sum_{x\in [n]} |\sigma(x) -\pi(x)|,
\end{align*}
where positions are as defined in~\eqref{eq:position}.
The Spearman footrule distance for partial rankings is a $2$-approximation for the Kemeny distance~\cite{fagin2004comparing}.

The notion of a distance between two rankings has an important extension in terms of a distance between a ranking and a set of rankings, which we refer to as rank-set distances. We focus our attention on two types of rank-set distances, defined below. For compactness, we use $\star$ to denote an arbitrary distance on pairs of rankings, but focus our attention throughout the paper on 
$\star \in\{\tau, S, K, prS\}$. 
\begin{definition}
Suppose that $\pi$ is a ranking and that $\Sigma$ is a set of rankings. Given a distance between two rankings $d_\star(\cdot,\cdot)$, the median-$\star$ distance ($med-\star$) between $\pi$ and $\Sigma$ equals
\begin{align*}
d_{med-\star}(\pi, \Sigma)=\frac{1}{|\Sigma|}\sum_{\sigma\in\Sigma} d_{\star}(\pi,\sigma).
\end{align*}
\end{definition}
\begin{definition}
Suppose that $\pi$ is a ranking and that $\Sigma$ is a set of rankings. Given a distance between two rankings $d_\star(\cdot,\cdot)$, the min-$\star$ distance ($min-\star$) between $\pi$ and $\Sigma$ is defined as 
\begin{align*}
d_{min-\star}(\pi, \Sigma)=\min_{\sigma\in\Sigma} d_{\star}(\pi,\sigma).
\end{align*}
\end{definition}
We recall that the focal problem of this work is to find constant approximation algorithms for the \textbf{MinMax} rank aggregation problem, which reads as
\begin{align*}
\textbf{MinMax:} \quad \min_{\pi}\max_{k} \lambda_k \,d(\pi, \Sigma^k),
\end{align*}
where $d(\pi, \Sigma^k)$ is a $med-\star$ or $min-\star$ distance, with $\star \in\{\tau, S, K, prS\}$. In our future analysis we use $\lambda^*\triangleq\max_{k}\lambda_k$ and $\cM\triangleq\{k:\lambda_k=\lambda^*\}$. Furthermore, we let $\pi^*$ denote the argument of the optimal solution of the \textbf{MinMax} problem and let $W=\max_{k} \lambda_{k}\, d(\pi^*, \Sigma^k)$.

\section{Approximate \textbf{MinMax} Aggregation}

As previously pointed out, the \textbf{MinMax} problem under both the $med-\star$ and $min-\star$ can be shown to be NP-hard using the results of~\cite{bachmaier2015hardness}, which established hardness for the special case $m_k=1$ and $d(\cdot, \cdot)$ a pseudometric. 
We hence focus on devising approximation algorithms for the \textbf{MinMax} problem. 
\subsection{Permutations}
We first consider ordinal data of the form of permutations. We show that a simple algorithm, which we term \emph{Pick-Rnd-Perm}, can achieve a $2$-approximation in expectation for the case of the $med-\star$ problem whenever $d_\star(\cdot,\cdot)$ is a pseudometric. Then, for $\star\in\{\tau, S\}$, we describe two $2$-approximation algorithms that use a combination of convex optimization and rounding procedures and offer significantly better empirical performance than random selection. Finally, we describe a $2$-approximation algorithm for the $min-\star$ problems when $d_\star(\cdot,\cdot)$ is a pseudometric. The selection algorithm essentially transforms the $min-\star$ problem into a $med-\star$ problem: Thus, the algorithms developed for approximating multiclass 
$med-\star$ problems may be used to approximate corresponding instances of the $min-\star$ problem. 

\textbf{The Pick-Rnd-Perm Algorithm.} Pick a permutation $\pi$ from $\cup_{k\in\cM}\Sigma_k$ uniformly at random. 
\begin{theorem} 
For the $d_{med-\star}(\cdot,\cdot)$ distance, where $d_\star$ is a pseudometric, the Pick-Rnd-Perm algorithm produces a $2$-approximation of the $med-\star$ problem. 
\end{theorem}
\begin{proof}
For a given $k$, 
\begin{align*}
&\lambda_k d_{med-\star}(\pi,\Sigma^k)=\frac{\lambda_{k}}{m_k}\sum_{i=1}^{m_k} d_{\star}(\pi,\Sigma^k)\\
\leq&  \lambda_{k}\left[d_{\star}(\pi, \pi^*) +d_{m-\star}(\pi^*,\Sigma^k)\right] \leq  \lambda^* d_{\star}(\pi, \pi^*)  + W.
\end{align*}
By calculating the expectation, we obtain 
\begin{align*}
\bE[\max_k\lambda_{k}d_{med-\star}(\pi,\Sigma^k)] \leq \lambda^*\bE[d_{\star}(\pi, \pi^*)]+W\leq 2W.
\end{align*}
\end{proof}
\vspace{-0.15in}
Clearly, random selection may be improved by picking the optimal permutation from $\cup_{k\in\cM}\Sigma_k$ instead. We term this approach Pick-Opt-Perm. Although the Pick-Rnd (Opt) -Perm algorithms are exceptionally simple and offer a $2$-approximation to the optimal solution, they have a number of drawbacks, including the fact that the aggregate is a given ranking from the clusters, which violates fairness rules of aggregates, and that its empirical performance is typically very poor. To mitigate these problems, we propose more sophisticated aggregation algorithms for both the $med-\tau$ and $med-S$ problems. 

\textbf{Case I: $d=d_{med-\tau}$}. For $C=1$, a well known method termed random pivoting proposed by Ailon et al.~\cite{ailon2010aggregation,ailon2008aggregating} offers a $2$-approximation in expectation for both the permutation and partial rank aggregation problem. In random pivoting, at each step, one element in the ranking is chosen uniformly at random and the remaining elements are partitioned based on the pairwise comparison with the pivot element. However, for the case of the \textbf{MinMax} problem with $C>1$, random pivoting may be inadequate: The difficulty lies in the fact that rankings in different classes may lead to widely disparate pairwise pivot comparisons. Another problem in this context is that while one may achieve a constant approximation in expectation for each class individually, the largest cost among classes may not be bounded due to the exchange of the expectation and maximization operators. Therefore, instead of pivoting, one must resort to a different approach to the problem. Our approach is to \emph{deterministically} round the fractional solution of a specific convex optimization problem. The deterministic rounding procedure is motivated by ideas in~\cite{van2007deterministic}. 

Let $w_{xy}^k\triangleq \frac{\lambda_k}{m_k}\sum_{i=1}^{m_k} \mathbf{1}\{\sigma_i^k(x)<\sigma_i^k(y)\}$, where $\mathbf{1}$ stands for the indicator function, and let $w_{xx}^k=0$ for all $x,k$. For a given ranking $\pi$, also define the variables $u_{xy}\triangleq \mathbf{1}\{\pi(x)<\pi(y)\}$. 
The \textbf{MinMax} problem may be stated as 
\begin{align}
\min_{\mathbf{u},q} &\quad\quad q \nonumber \\
\text{s.t.} & \quad \sum_{x,y\in[n]}  w_{xy}^ku_{yx} \leq q \quad \text{for all}\; k\in [C] \label{LPconstraint}\\
& \quad u_{xy}\in \{0,1\}, \nonumber \\
&\quad u_{xy}+u_{yx}=1 \quad \text{for all}\;  i,j \in [n],\, i\neq j \nonumber  \\
&\quad u_{xy}+u_{yz}+u_{zx}\geq 1 \quad \text{for all distincts}\;  x,y,z \in [n] \nonumber 
\end{align}
Note that if the rankings are permutations, then $w_{xy}^k+w_{yx}^k=\lambda_k,$ which is a value that only depends on $k$. 

The above integer program may be relaxed to a linear program by allowing $u_{xy}$ to take fractional values. Upon solving the linear program, one needs to round the values of $u_{xy}$. The next rounding procedure guarantees a $2$-approximation. 

Let $h_{xy}=\mathbf{1}_{u_{xy}\geq 1/2},$ if $x>y,$ and $h_{xy} = 1- h_{yx},$ if $x<y$. Let $v$ be a pivoting element for the rounding procedure and use $P_v(\mathbf{u})$ to denote the set of pairs of elements (excluding $v$) whose positions are determined by pivoting on $v$. Define
\begin{align*}
P_v(\mathbf{u})&=\{(x,y): x,\,y\in V_v,\, h_{vx}h_{yv}=1\}, \\
A_v^k(\mathbf{u})&=\sum_{x\in V_v}(h_{xv}w_{vx}^k+h_{vx}w_{xv}^k) +\sum_{(x,y)\in P_v}w_{xy}^k, \\
B_v^k(\mathbf{u})&=\sum_{x\in  V_v}(u_{xv}w_{vx}^k+u_{vx}w_{xv}^k) +\sum_{(x,y)\in P_v}(u_{xy}w_{yx}^k+u_{yx}w_{xy}^k).
\end{align*}
The rounding procedure makes iterative calls to the the following routine.
\vspace{-0.05in}
\begin{table}[htb]
\centering
\begin{tabular}{l}
\hline
\textbf{mmKT-Conv $(V,\mathbf{u})$}\\
\ 1: Choose the pivot $v\in V$ according to $v =\argmin_a \,\max_k \frac{A_a^k(\mathbf{u})}{B_a^k(\mathbf{u})}.$ \\
\ 2: Set $V_L=\emptyset, V_R=\emptyset$.\\
\ 3: For all $x\in  V_v$: \\
 \quad If $h_{xv}=1$, $V_L\leftarrow V_L\cup\{x\}$. Otherwise, $V_R\leftarrow V_R\cup\{x\}$.\\
\ 4: Return [\textbf{mmKT-Conv}$(V_L,\mathbf{u})$, $v$, \textbf{mmKT-Conv} $(V_R,\mathbf{u})$]. \\
\hline
\end{tabular}
\vspace{-0.05in}
\end{table} 
\vspace{-0.05in}
\begin{theorem}\label{Pivotthm}
The iterative application of the mmKT-Conv algorithm outputs a permutation with at most twice the cost of the optimal solution of the linear program~\eqref{LPconstraint}.
\end{theorem}
At each iteration of rounding, $A_v^k(\mathbf{u})$ denotes the cost of rounding incurred by the class $k$ of rankings, while $B_v^k(\mathbf{u})$ denotes the associated cost of the linear program for class $k$. Hence, the goal is to prove that for the given choice of the pivot $v$, we have $A_v^k(\mathbf{u})\leq 
2 B_v^k(\mathbf{u})$ for all $k\in [C]$. Suppose that $k'$ is the index of the class that maximizes $\frac{A_v^k(\mathbf{u})}{B_v^k(\mathbf{u})}$ at the first step of mmKT-Conv. Then, it suffices to show that $A_v^{k'}(\mathbf{u})\leq 2 B_v^{k'}(\mathbf{u})$. This result is a corollary of the following lemma. 
\begin{lemma}
$\sum_{v\in V} A_v^k(\mathbf{u})\leq 2\sum_{v\in V} B_v^k(\mathbf{u})$, $\forall \, k\in [C]$. 
\end{lemma}
\begin{proof}
To prove the claimed result, it suffices to prove that for any two distinct elements $x,y$, one has 
 \begin{align} \label{ineq1}
h_{xy}w_{yx}+h_{yx}w_{xy}\leq 2(u_{xy}w_{yx}+u_{yx}w_{xy}),
\end{align}
and for any triple of distinct elements $x,y,z$, one has 
 \begin{align} \label{ineq2}
\sum h_{xz}h_{zy}w_{yx}\leq 2\sum h_{xz}h_{zy}(u_{xy}w_{yx}+u_{yx}w_{xy}),
\end{align}
where the summation is circular over all permutations of $x, y, z$.
Both summations are taken over all possible permutations of the two (three) elements in the argument. 

The inequality \eqref{ineq1} is easy to prove: Suppose that $h_{xy}=1$. Then the sum on the left hand side equals $w_{yx}\leq 2u_{xy}w_{yx}$ which is bounded by the right hand side expression. To prove the inequality \eqref{ineq2}, consider the six variables associated with $x,y,z$, namely $h_{xy},h_{yx},h_{xz},h_{zx},h_{yz},h_{zy}$. These variables may be partitioned into two classes, $\{h_{xy},h_{zx},h_{yz}\}$ and $\{h_{yz},h_{xz},h_{zx}\}$. There are at least three variables that are 0's. Without loss of generality, suppose that the class $\{h_{xy},h_{zx},h_{yz}\}$ contains at least two 0's.\\
Case 1: Assume that $h_{xy},h_{zx},h_{yz}= 0$. Then, the difference of the left and right hand side of the inequality under consideration equals 
$$(1-2u_{xy})w_{yx}+(1-2u_{yz})w_{zy}+(1-2u_{zx})w_{xz}-$$
$$2u_{xz}w_{zx}-2u_{yx}w_{xy}-2u_{zy}w_{yz}.$$
The claimed result then follows from observing that $(1-2u_{xy})w_{yx}\leq u_{yx}(w_{yz}+w_{zx})$.\\
Case 2: Assume that $h_{xy}=1, h_{zx}, h_{yz}= 0$. The left hand side equals $w_{yx}\leq 2u_{xy}w_{yx}$ which is clearly bounded from above by the right hand side expression as $h_{xy}=1$.
\end{proof}
\textbf{Case II: $d=d_{med-S}$}. When $C=1$, the \textbf{MinMax} aggregation problem may be solved in polynomial time via weighted bipartite matching~\cite{dwork2001rankw}. However, when $C>1$, the problem is hard even if $m_k=1$ for all $k$~\cite{bachmaier2015hardness}. \\
\textbf{Step 1:} If we remove the integral constraint on the position of elements in $\pi$, the optimization problem of interest is convex and may be solved efficiently:
\begin{align}\label{convSF}
\mathbf{u}^*=\min_{\mathbf{u}\in\bR^n}\max_{k}\frac{\lambda_k}{m_k}\sum_{g=1}^{m_k}||\mathbf{u}-\sigma_g^k||_1, 
\end{align} 
where $||\mathbf{u}-\sigma_g^k||_1=\sum_{h\in[n]}|u(h)-\sigma_g^k(h)|$.\\
\textbf{Step 2 (mmSP-Conv):} We assign positions to elements according to the fractional solution $\mathbf{u}^*$ as follows. If $u^*(x)<u^*(y)$, we let $\pi(x)<\pi(y)$ for any two distinct elements $x,\,y,$ with ties broken randomly.

\begin{theorem}\label{SPthm} mmSP-Conv rounding increases the cost of the convex optimization problem~\eqref{convSF} at most twice.
\end{theorem}
\begin{proof}
First, we claim that the output of mmSP-Conv, denoted by $\pi_{S}$, is in $\Pi'\triangleq\{\pi'\in\bS^n:  ||\mathbf{u}^*-\pi'||_1=\min ||\mathbf{u}^*-\pi||_1\}$. This follows since for any ranking $\pi$, if two elements $x,\,y\in[n]$ satisfy $\pi(x)>\pi(y)$ and $u^*(x)<u^*(y)$, 
we may transpose $x$ and $y$ in $\pi$ to obtain a smaller $||\mathbf{u}^*-\pi||_1$. Second, for an arbitrary permutation $\sigma$, we have 
$$||\pi_{S}-\sigma||_1\leq ||\pi_{S}-\mathbf{u}^*||_1+||\sigma-\mathbf{u}^*||_1\leq 2||\sigma-\mathbf{u}^*||_1.$$
The claim follows by setting $\sigma=\sigma_i^k,$ $i\in[m_k],$ $k\in[C]$.
\end{proof}

Note that the integrality gap of the problems~\eqref{LPconstraint}~\eqref{convSF} is $2$, as one may consider two equally weighted classes, each of which contains one single ranking, $(1,2,3,4,...)$ and $(2,1,3,4,...)$, respectively. Hence, the best approximation constant via the use of $\mathbf{u}$ cannot be less than $2$, which implies that the proposed rounding is optimal. One may expect to achieve a \emph{smaller} approximation constant by outputting the better of the two results produced by Pick-Rnd-Perm and mmKT(SP)-Conv. This approach will be discussed in the full version of the paper.

We introduce next the min-Pick-Perm algorithm for solving the $d_{min-\star}$ problem.
\begin{table}[htb]
\centering
\begin{tabular}{l}
\hline
\textbf{min-Pick-Perm $(\Sigma^1,\Sigma^2,...,\Sigma^C)$, $(\lambda_1,\lambda_2,...,\lambda_C)$.}\\
\ 1: \textbf{For} each $k\in C$ and each ranking $\sigma_i^k\in\Sigma^k$ \\
\ 2: \quad Compute Score$_i^k=\max_{j\in C/\{k\}} \lambda_j \min_{\sigma_s^j\in \Sigma^j} d_{\star}(\sigma_i^k,\sigma_s^j).$\\
\ 3: Let $(i^*,k^*)=\arg_{(i,k)}\min\;$Score$_i^k$. Output $\pi=\sigma_{i^*}^{k^*}$.\\
\hline
\end{tabular}
\vspace{-0.07in}
\end{table} 
\vspace{-0.05in}
\begin{theorem}
If $d_{\star}$ is pseudometric, then min-Pick-Perm is a $2$-approximation algorithm for the $min-\star$ problems. 
\end{theorem}
\begin{proof}
By the definition of the $min-\star$ problem, each class contains at least one permutation, which without loss of generality we denote by $\sigma_1^k\in\Sigma^k,$ $k\in[C]$, that satisfies $\lambda_kd_{\star}(\pi^*,\sigma_1^k)\leq W$. As $d_{\star}$ is pseudometric, 
we have 
$$  \frac{\lambda_k\lambda_j}{\lambda_k+\lambda_j}d_{\star}(\sigma_1^k,\sigma_1^j)\leq W.$$
Next, choose an arbitrary $\tilde{k}\in\cM$ and let 
$k'=\arg\max_{j\in[C]/\{\tilde{k}\}}\lambda_{j}d_{\star}(\sigma_1^{\tilde{k}},\sigma_1^j)$. Then, 
\begin{align*}
&\min_{k\in[C]}\max_{j\in[C]/\{k\}}\lambda_{j}d_{\star}(\sigma_1^k,\sigma_1^j)\leq \lambda_{k'}d_{\star}(\sigma_1^{\tilde{k}},\sigma_1^{k'})\\
&\leq \frac{2\lambda_{\tilde{k}}\lambda_{k'}}{\lambda_{\tilde{k}}+\lambda_{k'}}d_{\star}(\sigma_1^{\tilde{k}},\sigma_1^{k'})\leq 2\max_{k,j} \frac{\lambda_k\lambda_j}{\lambda_k+\lambda_j}d_{\star}(\sigma_1^k,\sigma_1^j).
\end{align*}
Moreover, the output $\pi$ of min-Pick-Perm satisfies 
\begin{align*}
&\max_{j\in[C]}d_{\min-\star}(\pi,\Sigma^j)=\min_{k\in[C]}\min_{\sigma_i^k\in \Sigma^k}\max_{j\in[C]/{k}}\min_{\sigma_g^j\in\Sigma^j} d_{\star}(\sigma_i^k,\sigma_g^j)\\
&\leq \min_{k\in[C]}\max_{j\in[C]/\{k\}}\lambda_{j}d_{\star}(\sigma_1^k,\sigma_1^j).
\end{align*}
The result follows by combining the above inequalities.
\end{proof}

\begin{remark}
Let $(i^*,k^*)$ be the optimal indices generated by min-Pick-Perm. Define $\tilde{\Sigma}^{k^*}=\{\sigma_{i^*}^{k^*}\}$ and let
\begin{align*}
\tilde{\Sigma}^{j}=\{\sigma\in\Sigma^j: d_{\star}(\sigma_{i^*}^{k^*},\sigma)=d_{min-\star}(\sigma_{i^*}^{k^*},\Sigma^j)\}
\end{align*}
for $j\in[C]/\{k^*\}$. A $c-$approximate solution for the $med-\star$ problem with input $\{\tilde{\Sigma}_{k}\}_{k\in C}$, denoted by $\pi'$, satisfies 
\begin{align*}
&\max_{j\in[C]}\lambda_j d_{min-\star}(\pi',\Sigma^j)\leq \max_{j\in[C]}\lambda_j d_{med-\star}(\pi',\tilde{\Sigma}^j) \\
\leq &\;c\min_{\pi}\max_{j\in[C]}\lambda_j d_{med-\star}(\pi',\Sigma^j)\leq c\max_{j\in[C]}\lambda_j d_{med-\star}(\sigma_{i^*}^{k^*},\Sigma^j)\\
\leq & \;2cW. 
\end{align*}
Hence, $\pi'$ is a $2c-$approximation for the original $min-\star$ problem. Therefore, convex optimization and rounding can be used on the $med-\star$ problem. We refer to these adapted algorithms as min-mmKT-Conv and min-mmSP-Conv. 
\end{remark}

\subsection{Partial rankings}
All the algorithms proposed for permutation aggregation generalize to partial ranking aggregation. 
One may easily show that as long as the distance $d_{\star}$ defined for partial rankings is a pseudometric (e.g., $\star\in\{K,prS\}$), 
the $2$-approximation guarantees for all previous methods hold. 
To get a fractional solution in the program of mmKT-Conv, we have to change the constraint~\eqref{LPconstraint} to
\begin{align*}
\quad \frac{1}{2}T_k+\sum_{x,y\in[n]}  w_{xy}^ku_{yx} \leq w \quad \text{for all}\; k\in [C],
\end{align*} 
$$T_k=\frac{1}{m_k}\sum_{i=1}^{m_k}\sum_{1\leq x<y\leq n}\textbf{1}(\sigma_i^k(x)=\sigma_i^k(y)),$$ 
which does not depend on the type of output ranking. Also, note that $w_{xy}^k$ for partial rankings does not satisfy 
the equality $w_{xy}^k+w_{yx}^k=\lambda_k$, although the triangle inequality $w_{xy}+w_{yz}\geq w_{xz}$ still holds. 
As the proof of Theorem~\ref{Pivotthm} only requires the later inequality, the same rounding procedure offers a $2$-approximation. Also, in the optimization problem~\eqref{convSF} one has to use the definition $\sigma(x)$ for partial rankings. 
\vspace{-0.1in}
\section{Simulations}
We compare the performance of three families of algorithms: Convex optimization procedures with rounding (mmKT-Conv, mmSP-Conv, min-mmKT-Conv, min-mmSP-Conv), permutation selection (Pick-Rnd-Perm, Pick-Opt-Perm, min-Pick-Perm) and algorithms used for traditional min-median rank aggregation (FASLP-Pivot~\cite{ailon2008aggregating} and SP-Matching~\cite{dwork2001rankw}). The comparison shows that algorithms based on convex optimization yield significantly better results than naive selection methods, and that traditional aggregation algorithms are poor candidates for solving MinMax problems. 

First, we evaluate the proposed algorithms on synthetic data. The synthetic data is generated based on what we call a two-level Mallows model: First, we generate the permutations $\{\sigma^1,...,\sigma^C\}$ independently based on the Mallows distribution $\bP(\sigma^k)\propto\phi_1^{d_{\tau}(\sigma^k,e)}$~\cite{mallows1957non}. Then, for each class $k\in[C]$, we generate $m_k$ permutations $\sigma_1^k,...,\sigma_{m_k}^k$ independently according to the Mallows distribution $\bP(\sigma_i^k)\propto\phi_2^{d_{\tau}(\sigma_i^k,\sigma^k)}$. We set the number of classes to $C=3$, fix $\phi_2=0.7$ and let each class contain $m_k=10$ permutations. To control the distance between different classes, we choose $\phi_1$ from $\{0.5,0.7,0.9,1.0\}$. The objective function values for $100$ independent samples, obtained by different algorithms, are shown Table~\ref{syndata}.
\begin{table} 
\caption{Comparison of rank aggregation methods: Objective value (standard deviation)}\vspace*{-.05in}
\label{syndata}
       \begin{tabular}{c@{\hspace*{.15in}}|c|c|c|c}
       \multicolumn{5}{l}{A. $d_{med-\tau}$}  \vspace*{.01in} \\
       	  $\phi_1$ 		& 0.5 	& 0.7  & 0.9 & 1.0 \\
	  \hline
      	mmKT-Conv  	&14.5 (1.1) & 16.3 (1.4)  & 17.8 (1.3) & 17.9 (1.5)  \\
	Pick-Rnd-Perm   &17.8 (1.4) & 19.9 (2.1)   & 21.5 (1.8) & 21.6 (2.1) \\
   	Pick-Opt-Perm  	&15.9 (1.8)   & 18.1 (1.8)   & 20.0 (1.7) & 20.0 (1.6) \\
	FASLP-Pivot      &15.3 (1.4)  & 17.7 (2.1)  & 19.4 (2.2) & 19.7 (2.3) \\
		\vspace*{.02in}
      \end{tabular}
\begin{tabular}{c@{\hspace*{.15in}}|c|c|c|c}
 \multicolumn{5}{l}{B. $d_{med-S}$}  \vspace*{.01in} \\
       	  $\phi_1$ 		& 0.5 	& 0.7  & 0.9 & 1.0 \\
	  \hline
      	mmSP-Conv  	&23.3 (1.7) & 26.0 (2.1)  & 28.1 (2.3) & 28.4 (2.2)  \\
	Pick-Rnd-Perm   &27.0 (2.4) & 29.9 (2.7)   & 32.4 (2.7) & 32.1 (2.6) \\
   	Pick-Opt-Perm  	&24.5 (1.9)   & 27.5 (2.4)   & 29.9 (2.3) & 29.9 (2.1) \\
	SP-Matching     &26.3 (3.0)  & 30.5 (3.6)  & 35.3 (3.5) & 35.9 (3.4) \\
	\vspace*{.02in}
      \end{tabular}
\begin{tabular}{c@{\hspace*{.15in}}|c|c|c|c}
\multicolumn{5}{l}{C. $d_{min-\tau}$}  \vspace*{.01in} \\
       	  $\phi_1$ 		& 0.5 	& 0.7  & 0.9 & 1.0 \\
	  \hline
      	min-mmKT-Conv  	&6.9 (1.9) & 8.6 (2.3)  & 9.8 (2.6) & 10.0 (2.4)  \\
	min-Pick-Perm   &8.4 (1.7) & 10.5 (1.9)   & 11.8 (2.2) & 12.0 (1.8) \\
	FASLP-Pivot      &9.3 (1.9)  & 11.1 (2.2)  & 12.9 (2.7) & 13.1 (2.2) \\
	\vspace*{.02in}
      \end{tabular}
    \begin{tabular}{c@{\hspace*{.15in}}|c|c|c|c}
    \multicolumn{5}{l}{D. $d_{min-S}$}  \vspace*{.01in} \\
       	  $\phi_1$ 		& 0.5 	& 0.7  & 0.9 & 1.0 \\
	  \hline
      	min-mmSP-Conv  	&11.9 (2.6) & 14.1 (2.8)  & 16.1 (3.6) & 16.3 (3.1)  \\
	min-Pick-Perm   &13.9 (2.4) & 16.7 (2.7)   & 18.9 (3.0) & 18.9 (2.5) \\
   	SP-Matching      &17.1 (3.5)  & 22.4 (4.3)  & 26.2 (4.2) & 27.2 (4.0) \\
      \end{tabular}
      \vspace*{-.15in}
\end{table}

Our next test example comes from evolutionary biology, and is concerned with Mitochondrial DNA (mtDNA) genome aggregation. The aggregate in this case corresponds to an ancestral genome. The most common used rearrangement distance between two nuclear genomes is based on reversals~\cite{tesler2002efficient}, but mitochondrial DNA rearrangement studies have also involved the Kendall $\tau$ distance~\cite{chaudhuri2006tandem}. In the latter case, the authors only considered the median problem $C=1$, although  the min-max problem is equally relevant~\cite{dinu2012efficient,bachmaier2015hardness}. In our experiment, we used the mtDNA dataset from~\cite{bourque2002genome}. The dataset contains $11$ metazoan genomes with $36$ gene-blocks in some arrangement. We removed the ``signs'' of gene orders and let each genome represent one class, so that $C=11$ and $m_k=1$ for all $k$; we fixed $\lambda_k=1$. Table~\ref{realdata} shows the results. Due to page limitations, we relegate the significantly more space consuming empirical study of weighted multiclass mtDNA aggregation to the extended version of the paper. 

\begin{table} 
\caption{Mitochondrial DNA (mtDNA) aggregation}
\label{realdata}
 \vspace*{-.1in}
       \begin{tabular}{c@{\hspace*{.15in}}|c|l }
       	   & $d_{med-\tau}$ & Aggregated Sequences \\
	  \hline
	   \multirow{3}{*}{mmKT-Conv} & \multirow{3}{*}{210} &  1    10     7     2    17    12    30     9    11    23    19    20    21  \\
	   		&	& 13  35     3    15    14    25    26     6    16    32    28    34      \\   
			&	& 4 24  27	     18    36    29    31     8    33    22     5 \\
			\hline
	\multirow{3}{*}{Pick-Opt-Perm} & \multirow{3}{*}{267} &  1    27     2    17    36    20     3    29    10    11    35    12    30 \\
	&  & 21     9    19    18    28    33     7     8    16    26    14    34    13\\
	& & 24    15    32    25     4    22    23     6    31     5 \\
	\hline
	\multirow{3}{*}{FASLP-Pivot} & \multirow{3}{*}{269} &  1     2    17     7    23    12     3    20    30    21     6     9    10 \\
	&  & 11    15    19    28    25    27    18    32     8    33    24    13    34\\
	& & 14     4    35    29    26    16    36    31    22     5 \\
      	\hline
      \end{tabular}
 \vspace*{-.15in}
\end{table}

\bibliographystyle{IEEEtran}
\vspace{-0.1in}
\bibliography{HWbib}
\end{document}